\documentclass{article}

\usepackage[final,main]{neurips_2026}

\usepackage[utf8]{inputenc}
\usepackage[T1]{fontenc}

\usepackage{url}
\usepackage{booktabs}
\usepackage{amsfonts}
\usepackage{amsmath}
\usepackage{amssymb}
\usepackage{array}
\usepackage{microtype}
\usepackage{xcolor}
\usepackage{graphicx}
\usepackage{amsthm}
\usepackage{algorithm}
\usepackage{algpseudocode}
\usepackage{tcolorbox}

\usepackage{float}
\usepackage{enumitem}

\usepackage{caption}
\usepackage[hidelinks]{hyperref}
\hypersetup{
  pdftitle={EditDistill: Is It Possible to Guide Video Editing with Image Editing?},
  pdfauthor={Guojun Lei, Hong Li, Hongbing Yang, Lixue Gong, Chi Wang, Zheng Dong}
}

\algrenewcommand\algorithmicrequire{\textbf{Input:}}
\algrenewcommand\algorithmicensure{\textbf{Output:}}
\algnewcommand{\Initialize}{\State \textbf{Initialize:}}
\newtheorem{proposition}{Proposition}
\newtheorem{remark}{Remark}

\title{\method: Is It Possible to Guide Video Editing with Image Editing?}
\author{%
    Guojun Lei\textsuperscript{1*},\quad
    Hong Li\textsuperscript{4*},\quad
    Hongbing Yang\textsuperscript{3}\\[2pt]
    Lixue Gong\textsuperscript{1},\quad
    Chi Wang\textsuperscript{1},\quad
    Zheng Dong\textsuperscript{2}\\[6pt]
    {\normalfont\small\textsuperscript{1} State Key Lab of CAD\&CG, Zhejiang University}\\
    {\normalfont\small\textsuperscript{2} Zhejiang Sci-Tech University}\\
    {\normalfont\small\textsuperscript{3} Tsinghua University\qquad
    \textsuperscript{4} Beihang University}
}

\newcommand{\method}{\textsc{EditDistill}}
\newcommand{\R}{\mathbb{R}}
\definecolor{stageblue}{RGB}{218,235,252}
\definecolor{stagecyan}{RGB}{214,244,242}
\definecolor{stagegreen}{RGB}{224,244,223}
\definecolor{stageorange}{RGB}{253,234,214}
\definecolor{stageyellow}{RGB}{249,242,202}
\definecolor{linegray}{RGB}{80,80,80}
\definecolor{accentblue}{RGB}{54,99,170}
\definecolor{accentgreen}{RGB}{64,128,96}

\newtheorem{theorem}{Theorem}[section]

\newtheorem{definition}[theorem]{Definition}

\newcommand{\vfield}{v_{\theta}}
\newcommand{\vfedit}{v_{\theta}^{\mathrm{edit}}}
\newcommand{\xV}{x_t^{V}}

\newcommand{\Ldino}{\mathcal{L}_{\mathrm{DINO}}}
\newcommand{\Ledit}{\mathcal{L}_{\mathrm{edit}}}
\newcommand{\Lguide}{\mathcal{L}_{\mathrm{guidance}}}
\newcommand{\Lfm}{\mathcal{L}_{\mathrm{FM}}}

\newcommand{\Vs}{\mathbf{V}^{s}}
\newcommand{\Ve}{\mathbf{V}^{e}}
\newcommand{\DeltaF}{\Delta F^{\mathrm{edit}}}

\begin{document}
\raggedbottom

\maketitle
\begingroup
  \renewcommand{\thefootnote}{\fnsymbol{footnote}}
  \footnotetext[1]{Equal contribution.}
\endgroup

\begin{figure}[!htbp]
    \centering
    \includegraphics[width=\textwidth]{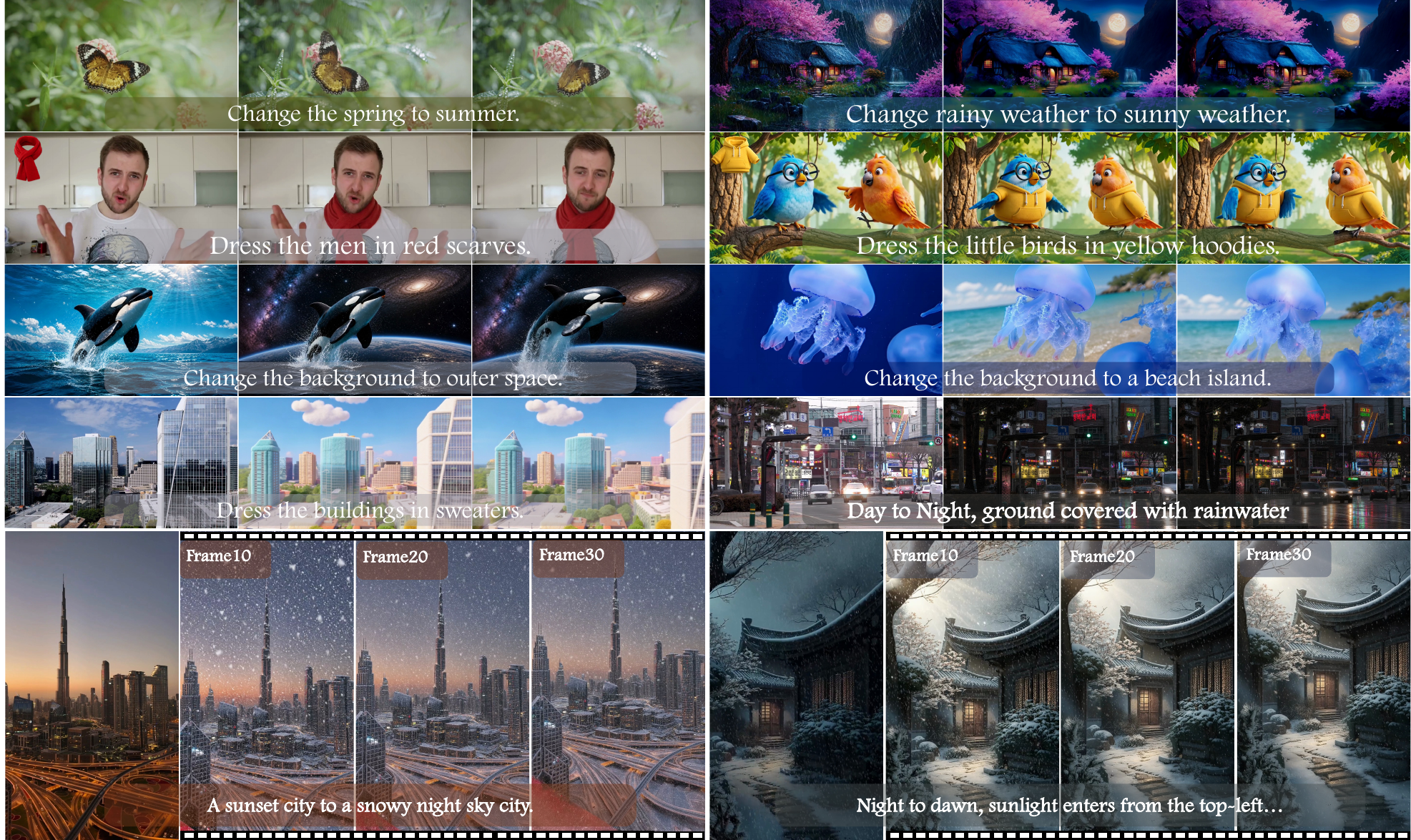}
    \caption{
    \method{} distills an image edit pair into a compact transformation descriptor and injects it into a video diffusion model, enabling local object edits, stylization, background replacement, and global relighting while preserving source motion and temporal coherence.
    }
    \label{fig:teaser}
\end{figure}

\begin{abstract}
\vspace{-2mm}
\label{sec:abstract}

Recent advances in image editing have enabled fine-grained and highly controllable visual modifications.
However, achieving similar levels of controllability in video editing remains an open challenge, particularly for precise local edits, dynamic temporal effects, and consistent lighting across frames.
In this paper, we propose \method, a framework that distills image-level visual transformations from edits to guide video editing.
Specifically, given a source video and an edit prompt, we first generate an edited version of the initial video frame using an off-the-shelf image editing model. We then distill the visual transformation of this original-edited image pair into a compact representation termed an \emph{edit descriptor}.
This descriptor is injected into a video diffusion model through our lightweight fusion module, which steers the denoising process toward the desired edit. To align the video trajectory with the demonstrated image edit while preserving motion and layout, we train with edit-descriptor alignment and DINO structural consistency losses, and derive an equivalent inference-time guidance formulation under flow matching.
Extensive experiments demonstrate that our method achieves strong performance on standard benchmarks and produces high-quality video editing results across diverse scenarios.
For more details and videos, please refer to the project webpage: \url{https://yu-shaonian.github.io/EditDistill-Web/}.

\end{abstract}

\section{Introduction}
\label{sec:intro}

Recent advances in AI-generated content (AIGC) have led to remarkable progress in both image and video generation. In particular, modern image editing models~\citep{zhao2025unified, wu2025qwen, LongCat-Image, team2025zimage, openai2025gptimage1, openai2026gptimage2, deepmind2025nanobananapro} have demonstrated strong capabilities in performing fine-grained and controllable edits, enabling complex local modifications with high fidelity. 
However, controllable video editing remains significantly more challenging, with existing methods suffering from three interrelated limitations.
First, fine-grained local editing, such as modifying object details or attributes, is often unreliable and typically requires explicit masks~\citep{hu2025hunyuancustom}, which are difficult to obtain and maintain across frames. Second, modeling dynamic background effects, such as rain, snow, or other temporally evolving phenomena, remains challenging~\citep{wu2023tuneavideo} largely due to the lack of explicit temporal control signals. A similar issue was also observed in frame-wise editing methods such as I2VEdit~\citep{ouyang2024i2vedit}, which ignore the denoising process alignment during video editing.
Third, existing methods struggle to jointly edit foreground and background in a consistent and controllable manner, often leading to inconsistent results~\citep{jiang2025vace}. 
Taken together, these limitations suggest that directly controlling video editing in the video domain is inherently difficult.
In contrast, we observe that such complex edits are much easier to achieve in the image domain using modern image editing models. 
We hypothesize that a single edited reference frame can encode the editing intent, including the visual signature of dynamic effects, while the video diffusion prior supplies the temporal unfolding.
This motivates us to ask: 
\begin{tcolorbox}[colback=gray!10, boxrule=0.2pt, arc=2pt, left=6pt, right=6pt, top=1pt, bottom=1pt]
\textit{Can we leverage the strong controllability of modern image editing models, and distill this capability to guide fine-grained and temporally consistent video editing?}
\end{tcolorbox}

To this end, we propose \method, a two-stage framework that distills image-level visual transformations from edits into a compact edit descriptor and injects it into video diffusion model for controllable video editing.
In the first stage, we utilize powerful image editing models ({\it e.g.}, Qwen-Image~\citep{wu2025qwen} or GPT-Image-2~\citep{openai2026gptimage2}) to generate an edited version of the first frame of a given video, based on an edit prompt. Given the original image and the edited image as a pair, we learn a compact edit descriptor that captures the transformation between them. Specifically, this transformation is encoded into a 1024-dimensional latent feature, which serves as guidance for the subsequent video editing process.
In the second stage, we inject the edit descriptor into a video diffusion model (Wan 2.1 backbone~\citep{wan2025wan}) through our lightweight fusion module. Combined with the edit descriptor alignment loss, this imposes strong constraints on the denoising trajectory, guiding the model toward the desired editing direction during the denoising process (see Fig.~\ref{fig:direction}).
As a result, the model can perform fine-grained and controllable video editing without requiring explicit masks. To ensure temporal consistency, we further introduce a feature alignment strategy based on the DINO~\citep{caron2021dino}, 
which aligns frame-level features similarity of the generated video with original video, thereby preserving structural coherence across frames.

Extensive experiments demonstrate competitive edit fidelity and strong preservation, temporal consistency, and perceptual quality on standard benchmarks. Figure~\ref{fig:teaser} illustrates both local and global editing effects. In summary, we make the following contributions:

\vspace{-1mm}
\begin{itemize}[leftmargin=*, itemsep=1pt, parsep=0pt, topsep=2pt]
\item We introduce image-to-video edit distillation: an image pair demonstrating a visual transformation is compressed into a compact feature and reused as an explicit video editing condition.
\item We design a parameter-efficient transformation injection mechanism for large video diffusion transformers, combining spatial cross-attention, global FiLM modulation, and additive residual editing while keeping the 14B backbone frozen.
\item We formulate edit descriptor alignment and DINO structural similarity losses that jointly improve edit fidelity and temporal coherence, and relate them to the gradient guidance during denoising.
\item We curate a five-family video editing dataset and show that \method{} achieves competitive edit fidelity while improving preservation, temporal consistency, and perceptual quality over the evaluated baselines.
\end{itemize}

\begin{figure*}[t]
    \centering
    \vspace{-4.0mm}
    \centerline{\includegraphics[width=1.0\linewidth]{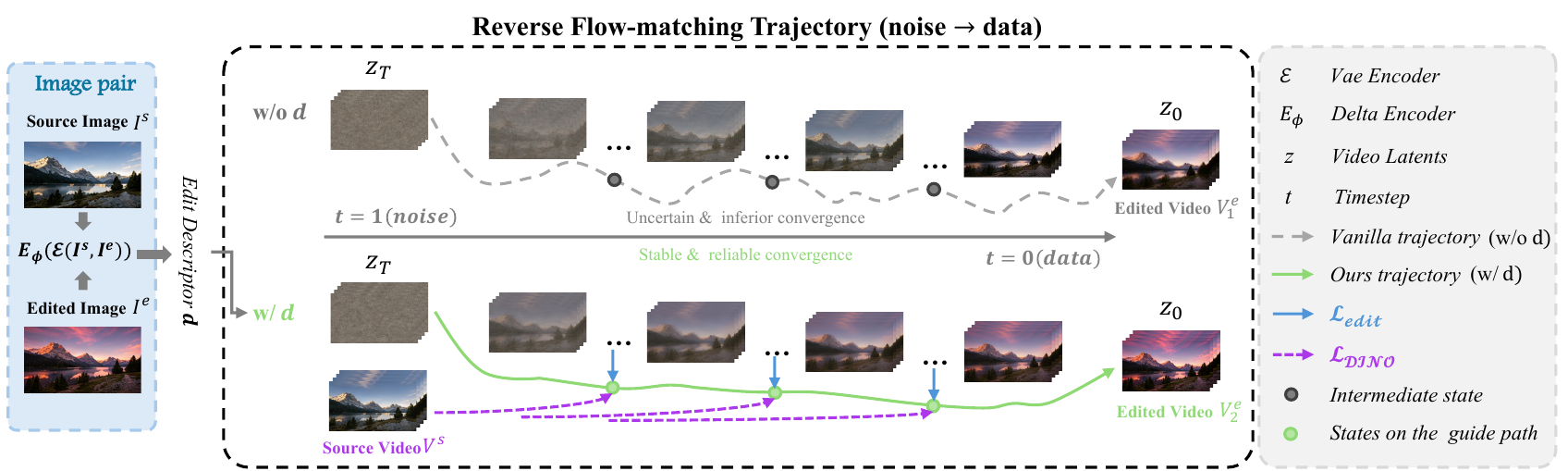}}
    \caption{\method{} encodes the visual transformation between image editing pair into a compact edit descriptor, and combines the edit descriptor alignment constraint with the DINO structural similarity loss to distill the image editing capability into the denoising trajectory of video editing.}
    \vspace{-1.5mm}
    \label{fig:direction}
\end{figure*}

\section{Related Work}
\label{sec:related}

\paragraph{Image editing with diffusion models.}
Recent image editing systems have demonstrated that diffusion models can follow rich editing instructions from text or images. Representative approaches include stochastic editing via noisy image trajectories \citep{meng2022sdedit}, prompt-based editing \citep{hertz2023prompt,brooks2023instructpix2pix}, inversion-based control \citep{mokady2023nulltext}, and latent diffusion models that enable high-resolution synthesis at scale \citep{rombach2022ldm}. While these methods achieve strong performance in single-image editing, their control signals are typically defined at the image level, limiting their ability to ensure temporal consistency or to reuse learned transformations across videos.
In contrast, our first stage represents an image edit as a compact feature-level transformation, enabling it to serve as a guidance signal for video editing.

\vspace{-1.5mm}

\paragraph{Video generation and video editing.}
Video diffusion models provide strong priors for text-to-video and image-to-video synthesis~\citep{blattmann2023stablevideo,wu2023tuneavideo,khachatryan2023text2videozero,guo2023animatediff,wan2025wan}. Building controllable editors on top of these priors remains challenging. Existing approaches explore attention control, feature propagation, inversion, correspondence guidance, and one-shot tuning, including Video-P2P~\citep{liu2024video}, TokenFlow~\citep{geyer2024tokenflow}, CoDeF~\citep{ouyang2024codef}, VideoSwap~\citep{gu2024videoswap}, AnyV2V~\citep{ku2024anyv2v}, and RF-Solver-Edit~\citep{wang2025rfsolveredit}. Recent systems broaden the control interface: HunyuanCustom~\citep{hu2025hunyuancustom} targets customized video generation from multimodal reference conditions, while VACE~\citep{jiang2025vace} unifies video creation and editing within a single framework. Motion-control methods such as MagicMotion~\citep{Li_2025_ICCV} inject trajectory conditions through a ControlNet-style branch and achieve effective object-motion control, but this design introduces a relatively heavy conditional pathway and assumes explicit trajectory or region-level controls. These methods work well when the desired change can be specified through text, inversion trajectories, reference conditions, or reliable correspondences, but they often struggle with exemplar-specific appearance changes, global relighting, or complex local edits without masks. \method{} instead replaces implicit prompt interpretation and dense control specification with an explicit learned transformation descriptor while relying on the video backbone for temporal rendering.
Compared with ControlNet-style residual branches~\citep{zhang2023controlnet}, and low-rank adapters~\citep{hu2022lora}, our injector follows this parameter-efficient design but focus more on the diffusion process.

\section{Problem Formulation}
\label{sec:problem}

\paragraph{Our setting and notation.}
Let $(I^s, I^e)$ be an image pair, where $I^s$ is the source image and $I^e$ is its edited version obtained under a given edit prompt $p$. Let $V^s = \{v_n^s\}_{n=1}^{N}$ be a source video of $N$ frames.
Our goal is to learn a video generation model $F$ that produces an edited video $V^e = F(V^s, [I^s, I^e], p)$ by applying a visual transformation consistent with the image-level $I^s \rightarrow I^e$, while maintaining temporal coherence with respect to the source video $V^s$.

\vspace{-2mm}

\paragraph{Two-stage factorization.}
Rather than learning a monolithic mapping from all inputs to the edited video, we introduce a compact edit descriptor $\mathbf{d}\in\mathbb{R}^{m}$ and factorize the problem as
\begin{equation}
{\bf d} \;=\; E_{\phi}(\mathcal E(I^s, I^e)), \ \ \ \  V^e \;=\; G_{\theta}(V^s, p, {\bf d}).
\label{eq:factorization}
\end{equation}
Here, $E_{\phi}$ is learned from image pairs and encodes the transformation between $(I^s, I^e)$ into a compact feature denotes edit descriptor within a additional VAE-framework network, while $G_{\theta}$ incorporates this feature into a pre-trained video diffusion model. This two-stage design separates transformation modeling from video generation, which enables the model to learn image-level edits and apply them consistently during video synthesis. As shown in the attention visualization in Fig.~\ref{fig:attention}, which shows that the model focuses on regions relevant to the edit intent, indicating that the learned transformation effectively localizes the editing process.

\section{Method : \method}
\label{sec:method}

\begin{figure*}[t]
    \vspace{-2em}
    \centering
    \centerline{\includegraphics[width=1.0\linewidth]{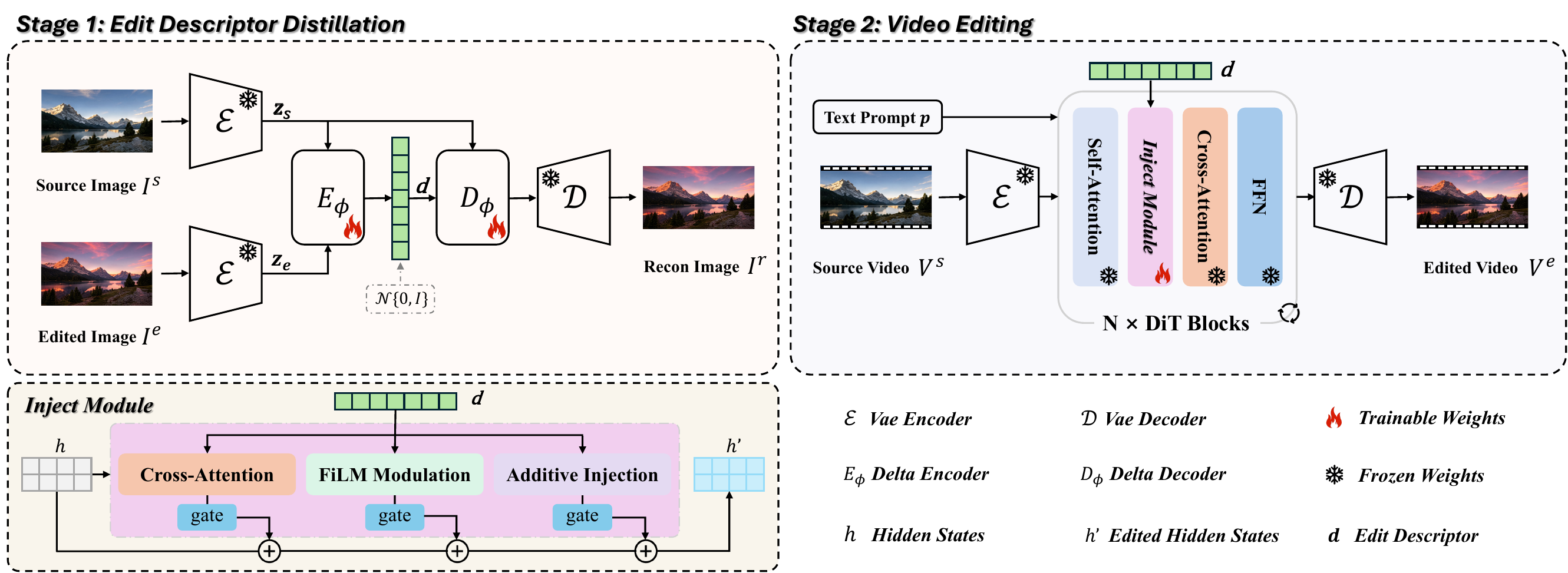}}
    \caption{\textbf{Overview of \method.} {\bf Stage~I} encodes the transformation between an image pair $(I^s, I^e)$ into the edit descriptor ${\bf d}$. {\bf Stage~II} keeps the {\it Wan2.1-I2V-14B} video diffusion frozen and trains the inject modules to inject ${\bf d}$ into selected transformer blocks by lightweight feature fusion modules steering the diffusion denoising toward the desired edit to produce a temporally coherent video $V^e$ from the video $V^s$.}
    \vspace{-3mm}
    \label{fig:editstill}
\end{figure*}

\method{} contains two components trained sequentially. Stage~I learns an image edit descriptor that captures the transformation between a before/after image pair. Stage~II uses this descriptor to condition a video diffusion model. The design deliberately separates edit extraction from temporal rendering: Stage~I answers what visual change should be applied, and Stage~II decides where and how that change should appear in each video frame.
This design further enables parameter-efficient adaptation by {\bf training only the lightweight injection modules}, allowing the model to steer the denoising trajectory of the video diffusion toward the desired edit with a small number of trainable parameters (see pipeline in Fig.~\ref{fig:editstill}).

\subsection{Stage I: Distilling Image Edits into Compact Descriptors}
\label{sec:stage1}
We first encode the source and edited images into latent space features using a shared frozen VAE, and the features are denoted as $z^s = \mathcal{E}_{\mathrm{vae}}(I^s)$ and $z^e = \mathcal{E}_{\mathrm{vae}}(I^e)$. The pair $(z^s, z^e)$, together with their difference $(z^e - z^s)$, is fed into a multi-scale variational encoder (noted as $E_{\phi}$) to infer the compact edit descriptor ${\bf d}$. Specifically, the encoder predicts a posterior distribution $q_\phi(\mathbf{d} \mid z^s, z^e) = \mathcal{N}(\mu_\phi, \operatorname{diag}(\sigma_\phi^2))$ over a 1024-dimensional feature $\mathbf{d}$, which is obtained via reparameterized sampling during training and taken as the mean $\mu_\phi$ at inference. 
Then, a conditional decoder reconstructs the edited latent as a residual transformation applied to the source $\hat{z}^e = z^s + D_{\phi}(z^s, \mathbf{d})$, which enforces a decomposition where $z^s$ preserves the shared content and $\mathbf{d}$ captures the edit-specific variation. Stage~I is trained with the following objective:
\begin{equation}
\mathcal{L}_{\mathrm{Distill}} = \lambda_{1}\bigl\lVert \hat{z}^e - z^e \bigr\rVert_1 
+ \lambda_{2}\bigl\lVert \hat{z}^e - z^e \bigr\rVert_2^2 
+ \lambda_{\mathrm{kl}}\, \mathrm{KL}\bigl(q_{\phi} \,\Vert\, \mathcal{N}(0, I)\bigr),
\label{eq:stage1-loss}
\end{equation}
which combines latent reconstruction with a KL regularizer to encourage a smooth and compact transformation space.
After training, the encoder $E_\phi$ is frozen and used as a edit extractor in Stage~II. Additional architectural details and the trade-off between representation capacity and transferability are provided in Appendix~\ref{app:stage1-details}.

\begin{figure}[t]
    \centering
    \includegraphics[width=1\linewidth]{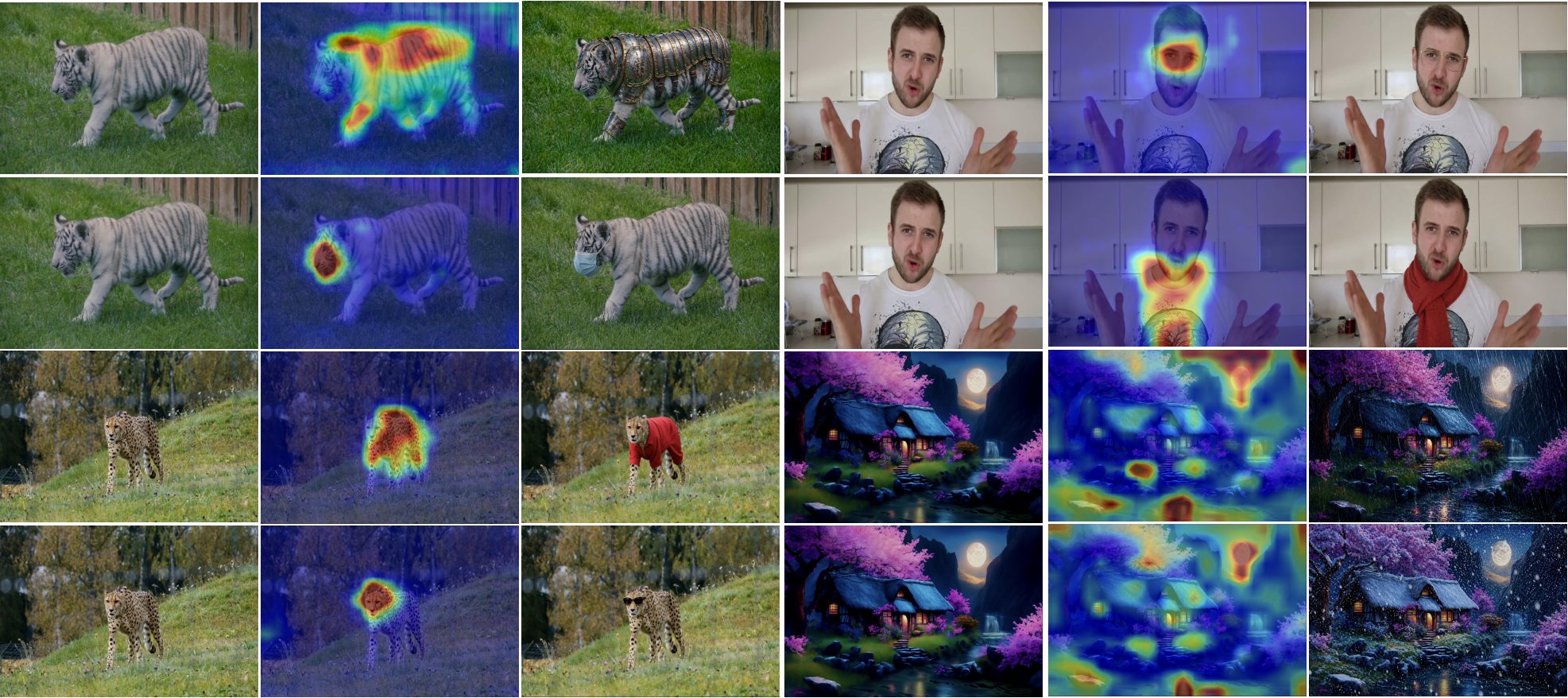}
    \caption{\textbf{Descriptor attention.} Cross-attention maps show that local edits concentrate on the regions that need to change, while global edits such as relighting or stylization produce broader but still structured responses. The descriptor behaves as a transformation operator rather than a fixed mask or copied target image.}
    \label{fig:attention}
    \vspace{-3.5mm}
\end{figure}

\paragraph{What does the edit descriptor $\mathbf{d}$ encode?}

The edit descriptor together with the residual reconstruction encourages a separation of roles: $z^s$ preserves the shared content, while $\mathbf{d}$ captures the edit-specific change. The attention maps in Fig.~\ref{fig:attention} help reveal what $\mathbf{d}$ encodes. 
For local edits, high responses concentrate on regions that need to change ({\it e.g.}, object parts or inserted attributes), while background remains largely unaffected. For global edits such as illumination or style changes, the responses become broader but still focus on meaningful regions like the sky, light sources, and main scene elements, rather than spreading uniformly. 
These patterns indicate that $\mathbf{d}$ does not represent a target image or a fixed mask. Instead, it specifies \emph{what transformation to apply}, while its interaction with video features determines \emph{where} the change is applied in each frame as illustrated by the frame-wise descriptor attention maps in Fig.~\ref{fig:app-attention-video}.
This separation allows the edit to transfer consistently across frames without explicit mask supervision, while preserving unaffected regions.

\subsection{Stage II: Descriptor-Guided Video Diffusion}
\label{sec:stage2}
Stage~II renders the distilled descriptor into a temporally coherent edited video. We adopt {\it Wan2.1-I2V-14B} as the base generator, trained under Flow Matching~\citep{lipman2022flow} paradigm, which models a {\it time-dependent velocity field} $v_\theta(x_t, t, c)$ along the straight-line interpolation from data ($t=0$) to noise ($t=1$).

We extend this backbone with a proposed lightweight feature injection module to obtain a descriptor-conditioned velocity field $v_{\Theta}(x_t, t, [c_{\mathrm{text}}, c_{\mathrm{video}}], \mathbf{d}) \in \mathbb{R}^{N \times H \times W \times C}$,
where $\mathbf{d}$ is either inferred from $(I^s, I^e)$ or precomputed once and reused across videos. The backbone is kept entirely frozen, and only this feature fusion module is trained, enabling parameter-efficient adaptation.

Our feature injection and fusion module first adapts $\mathbf{d}$ via a shared MLP, denoted as $\bar{\mathbf{d}} = \mathrm{MLP}(\mathrm{LN}(\mathbf{d})) + \mathbf{d}$, and then injects it into the transformer blocks through three complementary branches: (1) a cross-attention branch for spatially localized edits, (2) a FiLM branch \citep{perez2018film} for global appearance control, and (3) an additive branch for low-rank edit biases.
A zero-initialized gated fusion combines the branches, so the wrapped model initially matches the frozen backbone and the edit pathway opens gradually during training. Appendix~\ref{app:injector-details} gives details.

\subsection{Training Objectives and Inference Guidance}
\label{sec:training}

The Stage~II objective combines the standard flow-matching loss with two auxiliary losses that couple the video output to (1) the desired image-level edit and (2) the temporal structure of the source video.

Under the Rectified Flow formulation, the model predicts a velocity field $v_{\Theta}(x_t, t, c, {\bf d})$ along the linear interpolation $x_t = (1 - t)x_0 + t\cdot \varepsilon$, where the optimal target is $\varepsilon - x_0$. The flow-matching loss is
\begin{equation}
\mathcal{L}_{\mathrm{FM}} = \mathbb{E}_{x_0,\varepsilon,t}\bigl\lVert v_{\Theta}(x_t, t, c, {\bf d}) - (\varepsilon - x_0)\bigr\rVert_2^2.
\end{equation}

A one-step estimate of the clean sample is obtained as $\hat{x}_0 = x_t - t\cdot v_{\Theta}(x_t, t, c, {\bf d})$, and a reference prediction $\hat{x}_0^{C}$ is defined by performing the same denoising step without descriptor conditioning. To explicitly supervise the editing guidance, we introduce two losses:
\begin{equation}
\mathcal{L}_{\mathrm{edit}} = \tfrac{1}{N}\!\sum_{n}\!\bigl\| E_{\phi}(\hat{x}_{0,n}, \hat{x}_{0,n}^{C}) - {\bf d} \bigr\|_2^2, \qquad
\mathcal{L}_{\mathrm{DINO}} = \tfrac{1}{N}\!\sum_{n}\!\bigl\| \Psi(\hat{x}_{0,n}) - \Psi(v_n^{s}) \bigr\|_2^2,
\label{eq:aux-losses}
\end{equation}
where $E_\phi(\cdot,\cdot)$ extracts the transformation between two frames, and $\Psi(\cdot)$ is a frozen DINO encoder~\citep{caron2021dino}. The first term enforces \emph{edit descriptor alignment} by matching the induced feature shift with ${\bf d}$, while the second enforces \emph{temporal consistency} by preserving the structural layout of the source video. Beyond loss-based training, these objectives can be interpreted as gradient-based guidance on the denoising trajectory:
\begin{equation}
    \vfedit(\xV,\,t)
    = \underbrace{\vfield(\xV,\,t,\,c)}_{\text{base denoising direction}}
    \;-\; t\!\left(
        \underbrace{
            \alpha_{\mathrm{edit}}\,\nabla_{\xV}\Ledit
        }_{\text{edit descriptor alignment}}
        +
        \underbrace{
            \alpha_{\mathrm{DINO}}\,\nabla_{\xV}\Ldino
        }_{\text{temporal consistency}}
    \right),
    \label{eq:final_vfield}
\end{equation}
which steers the denoising direction toward the desired edit while preserving structural consistency. The full objective is defined with weighting factors $\alpha_{\mathrm{edit}}$ and $\alpha_{\mathrm{DINO}}$, as:
\begin{equation}
\mathcal{L}_{\mathrm{Stage\text{-}II}} = \mathcal{L}_{\mathrm{FM}} + \alpha_{\mathrm{edit}}\,\mathcal{L}_{\mathrm{edit}} + \alpha_{\mathrm{DINO}}\,\mathcal{L}_{\mathrm{DINO}}.
\end{equation}
Moreover, through the score--velocity identity $\nabla_{x_t}\log p_t(x_t) = -v_\theta(x_t, t)/t$ (see Appendix~\ref{eq:score_velocity}), the same auxiliary objectives can be applied as test-time gradient guidance. This provides a unified formulation that supports both training-time conditioning and tuning-free inference-time control (details in Appendices~\ref{app:fm-derivation} and \ref{app:guidance-algorithm}).

\section{Experiments}
\label{sec:experiments}

We evaluate \method{} on different types of edits and several video benchmark datasets. Our experiments aim to answer four questions: (1) whether the image edit can be distilled to the video diffusion model, (2) whether the unedited content is preserved, (3) whether the edit is consistent over time, and (4) which parts of the method contribute most to the improvement.

\subsection{Experimental Setup}
\label{sec:exp-setup}

\paragraph{Training and inference details.}

Stage~I trains the variational transformation feature encoder $E_{\phi}$ and residual decoder $D_{\psi}$ in the frozen {\it Wan2.1} VAE latent space at a resolution of $480\times832$. We use AdamW with a learning rate of $1\times10^{-4}$, batch size $64$, and $200{,}000$ training iterations.
In Stage~II, we load the frozen Stage~I encoder together with the frozen {\it Wan2.1-I2V-14B} backbone, and train only the edit feature injection modules. Optimization is performed with AdamW using a learning rate of $1\times10^{-5}$, a global batch size of $16$, and $30{,}000$ iterations.
Unless otherwise specified, injectors are inserted into every third transformer block of {\it Wan}. The loss weights are set to $\alpha_{\mathrm{edit}}=0.5$ and $\alpha_{\mathrm{DINO}}=0.2$. During inference, we generate $81$-frame video clips at $480\times832$ using $50$ flow-matching denoising steps.
All training experiments are conducted on $16$ NVIDIA H800 GPUs (80GB memory each). The trainable injection modules occupy approximately 480MB, accounting for less than $4\%$ of the 14B backbone parameters.

\vspace{-1.5mm}

\paragraph{Datasets.}
Stage I uses approximately 2M paired images aggregated from OmniGen2 and GPT-Image-Edit \citep{wu2025omnigen2,wang2025gptimageedit15mmillionscalegptgeneratedimage}. 
For stage II, We carefully constructed a large-scale video editing dataset. First, we leveraged the Wan 2.1 video generation model to synthesize preliminary video pairs from image editing pair data, while strictly enforcing identical text prompts except for the edited regions. To further improve the consistency of these coarse video-editing pairs, we use the edit loss and DINO feature loss described in the Method section, which effectively constrained both semantic and structural alignment between these video pairs. After a rigorous filtering process, we ultimately obtained 100K high-quality video editing samples. In addition, for the training data, we also constructed paired samples from the OpenVid~\citep{nan2024openvid} dataset by randomly injecting noise into videos and randomly shuffling video frames, using the original real videos as reconstruction targets.

\vspace{-1.5mm}

\paragraph{Evaluation metrics.}

No single metric fully captures video editing quality, so we report complementary metrics. Edit fidelity is evaluated by CLIP-Edit Similarity (CLIP-E), defined as the similarity between the image edit $\mathrm{CLIP}(I^e)-\mathrm{CLIP}(I^s)$ and the average video edit direction $\mathrm{CLIP}(v_n^e)-\mathrm{CLIP}(v_n^s)$, and by CLIP-text alignment (CLIP-T) \citep{radford2021clip}. 
Preservation is measured by DINO-sim on regions outside the edit estimated by Grounded-SAM2 \citep{ren2024grounded,ravi2024sam2,oquab2023dinov2}. Temporal coherence is assessed using CLIP-based adjacent-frame consistency (Temp.) and optical-flow warp error (Warp) \citep{huang2024vbench}, while perceptual quality is measured by FVD \citep{unterthiner2019fvd}. 
We further report VBench-style metrics, including subject consistency (SubC), background consistency (BkgC), motion smoothness (MoS), aesthetic quality (AesQ), and dynamic degree (DyaD), for the consistency of video quality.
We compare with nine representative methods on VACE-benchmark covering frame-wise editing, feature propagation, inversion, motion-conditioned generation, reference conditioning, and unified video editing: prompt-only Wan2.1-I2V-14B \citep{wan2025wan}, InstructPix2Pix \citep{brooks2023instructpix2pix}, TokenFlow \citep{geyer2024tokenflow}, AnyV2V \citep{ku2024anyv2v}, RF-Solver-Edit \citep{wang2025rfsolveredit}, MOFA-Video \citep{niu2024mofa}, Omni-Video \citep{tan2025omnivideo}, HunyuanCustom \citep{hu2025hunyuancustom}, and VACE \citep{jiang2025vace}. When a baseline supports image references, we provide the edited first frame $I^e$; otherwise we provide the corresponding edit caption. All methods run at matched $480\times832$ resolution and $81$-frame inference budget using their public high-quality settings.

\subsection{Main Experimental Results}
\label{sec:exp-main}
\vspace{-1.5mm}
\paragraph{Edit fidelity, preservation, and perceptual quality.}
Tab.~\ref{tab:main-results} reports averages over the five edit families. \method{} achieves the second-best CLIP-E score of 0.245, close to VACE at 0.247, while its CLIP-T score of 0.282 is below the strongest baseline. It achieves the best DINO-sim, Temp., Warp, and FVD scores. Compared with VACE, \method{} improves DINO-sim from 0.701 to 0.738, reduces Warp from 3.42 to 2.97, and lowers FVD from 914.6 to 852.1. These results indicate competitive image-pair edit fidelity together with stronger structural preservation, temporal coherence, and perceptual quality.

\begin{table}[t]
\caption{\textbf{Quantitative comparisons.} All methods run at matched $480\times832$ and $81$-frame experimental settings. For evaluation metrics, CLIP-E and CLIP-T measure edit fidelity; DINO-sim measures preservation of unedited regions; Temp. and Warp measure temporal coherence; FVD measures perceptual video quality. Best results are in \textbf{bold} and second best results are \underline{underlined}.}
\vspace{2mm}
\label{tab:main-results}
\centering
\scriptsize
\setlength{\tabcolsep}{3.6pt}
\begin{tabular}{lcccccc}
\toprule
Method & CLIP-E$\uparrow$ & CLIP-T$\uparrow$ & DINO-sim$\uparrow$ & Temp.$\uparrow$ & Warp$\downarrow$ & FVD$\downarrow$ \\
\midrule
InstructPix2Pix~\citep{brooks2023instructpix2pix} & 0.221 & 0.288 & 0.534 & 0.892 & 7.15 & 1284.7 \\
TokenFlow~\citep{geyer2024tokenflow} & 0.196 & 0.275 & 0.668 & 0.948 & 4.21 & 1052.3 \\
AnyV2V~\citep{ku2024anyv2v} & 0.137 & 0.251 & 0.591 & 0.856 & 6.94 & 1485.2 \\
RF-Solver-Edit~\citep{wang2025rfsolveredit} & 0.224 & 0.286 & 0.654 & 0.926 & 4.58 & 996.3 \\
MOFA-Video~\citep{niu2024mofa} & 0.182 & 0.271 & 0.621 & 0.943 & 4.83 & 1024.1 \\
Prompt-only Wan2.1~\citep{wan2025wan} & 0.158 & 0.291 & 0.612 & 0.951 & 3.84 & 982.4 \\
Omni-Video~\citep{tan2025omnivideo} & 0.218 & 0.284 & 0.661 & 0.939 & 4.34 & 968.7 \\
HunyuanCustom~\citep{hu2025hunyuancustom} & 0.236 & \underline{0.295} & 0.687 & 0.949 & 3.71 & 941.2 \\
VACE~\citep{jiang2025vace} &\textbf{0.247} & \textbf{0.302} & \underline{0.701} & \underline{0.957} & \underline{3.42} & \underline{914.6} \\
\midrule
\textbf{\method{} (ours)} & \underline{0.245} & 0.282 & \textbf{0.738} & \textbf{0.964} & \textbf{2.97} & \textbf{852.1} \\
\bottomrule
\end{tabular}
\vspace{-2mm}
\end{table}

\vspace{-1.5mm}

\paragraph{Video consistency under dynamic motion.}
Tab.~\ref{tab:vbench-consistency} further evaluates video-level consistency. \method{} obtains the best BkgC, MoS, AesQ, and DyaD scores, while its SubC score of 0.942 is below Omni-Video at 0.964. The combination of strong background consistency and the highest dynamic degree suggests that the method maintains coherent edits under motion, although subject consistency remains a relative weakness.

\begin{table}[t]
\caption{\textbf{Comparisons of video editing consistency.} SubC: subject consistency; BkgC: background consistency; MoS: motion smoothness; AesQ: aesthetic quality; DyaD: dynamic degree. Best results are in \textbf{bold}; second best results are \underline{underlined}.}
\vspace{2mm}
\label{tab:vbench-consistency}
\centering
\scriptsize
\setlength{\tabcolsep}{4.2pt}
\begin{tabular}{lccccc}
\toprule
Method & SubC$\uparrow$ & BkgC$\uparrow$ & MoS$\uparrow$ & AesQ$\uparrow$ & DyaD$\uparrow$ \\
\midrule
InstructPix2Pix~\citep{brooks2023instructpix2pix} & 0.842 & 0.871 & 0.924 & 0.429 & 0.612 \\
TokenFlow~\citep{geyer2024tokenflow} & 0.946 & 0.951 & 0.972 & 0.514 & 0.524 \\
AnyV2V~\citep{ku2024anyv2v} & 0.812 & 0.768 & 0.912 & 0.431 & 0.512 \\
RF-Solver-Edit~\citep{wang2025rfsolveredit} & 0.928 & 0.939 & 0.967 & 0.571 & 0.628 \\
MOFA-Video~\citep{niu2024mofa} & 0.932 & 0.945 & 0.974 & 0.568 & \underline{0.718} \\
Prompt-only Wan2.1~\citep{wan2025wan} & 0.921 & 0.943 & 0.944 & 0.562 & 0.681 \\
Omni-Video~\citep{tan2025omnivideo} & \textbf{0.964} & 0.952 & 0.971 & 0.583 & 0.654 \\
HunyuanCustom~\citep{hu2025hunyuancustom} & 0.949 & 0.958 & \underline{0.978} & 0.594 & 0.692 \\
VACE~\citep{jiang2025vace} & \underline{0.954} & \underline{0.962} & 0.953 & \underline{0.612} & 0.671 \\
\midrule
\textbf{\method{} (ours)} & 0.942 & \textbf{0.969} & \textbf{0.981} & \textbf{0.645} & \textbf{0.724} \\
\bottomrule
\end{tabular}
\vspace{-2mm}
\end{table}

\begin{figure}[t]
    \centering
    \vspace{-1mm}
    \includegraphics[width=1\linewidth]{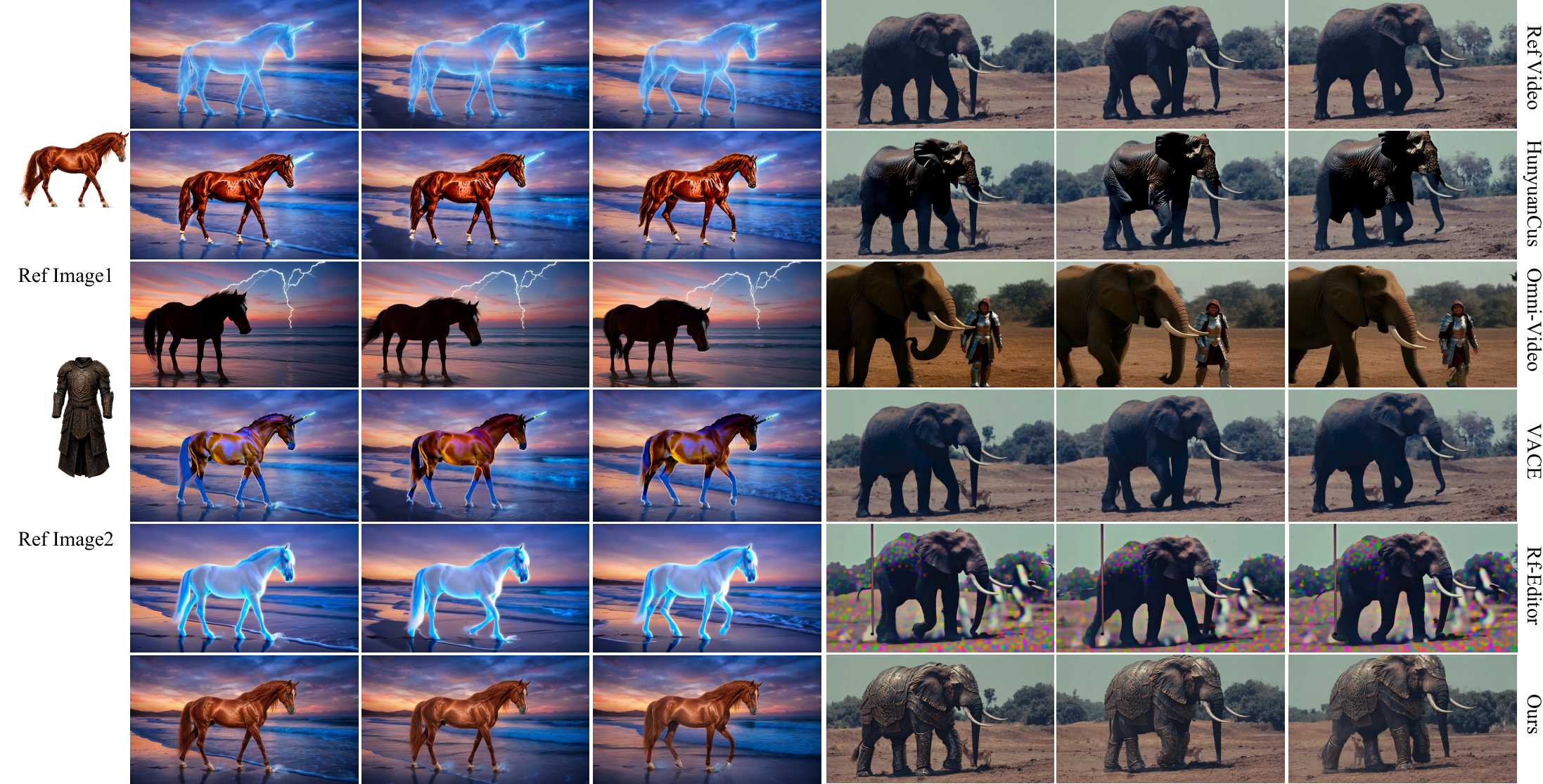}
    \caption{\textbf{Qualitative comparison.} \method{} performs both local and global edits while keeping the main subject, background, and motion stable. In contrast, other methods may miss the desired edit, unintentionally change nearby regions, or produce unstable results across frames.}
    \label{fig:compare}
    \vspace{-0.5mm}
\end{figure}

\vspace{-1.5mm}

\paragraph{Qualitative comparison.}
Fig.~\ref{fig:compare} provides a qualitative comparison that helps explain the differences behind the above metrics. Prompt-driven methods often capture the coarse caption but miss small exemplar-specific attributes, such as accessories or fine texture changes. Frame-wise and inversion-based methods can make a sharper local change in the first few frames, but the edit may jitter, weaken, or contaminate nearby regions as the clip evolves. Reference-conditioned methods better preserve the target appearance, yet they sometimes transfer irrelevant color or texture from the reference image into the background. \method{} is more stable because the image pair defines a transformation rather than a target appearance: the distilled transformation feature $\mathbf{d}$ indicates what should change, and cross-attention to video tokens determines where the change should be applied.

\begin{figure}[t]
    \centering
    \includegraphics[width=1\linewidth]{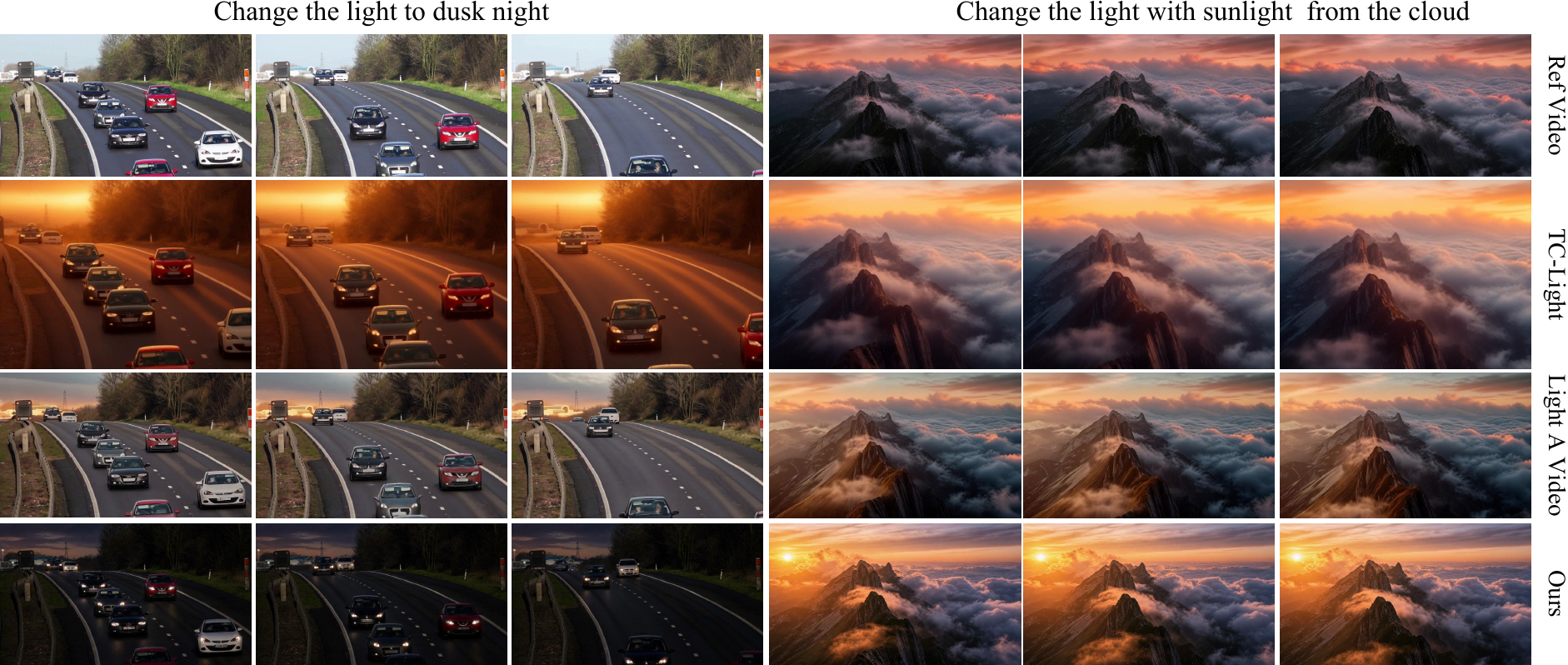}
    \caption{\textbf{Relighting comparison.} The distilled descriptor captures a global lighting transformation and guides coherent relighting across frames while preserving scene geometry and motion.}
    \label{fig:relight}
    \vspace{-3mm}
\end{figure}

Fig.~\ref{fig:relight} highlights the challenging setting of video relighting. Unlike local editing methods or reference-based appearance transfer approaches, relighting requires globally adjusting scene illumination while preserving object geometry, motion, and temporal consistency. Compared with existing methods, \method{} better captures the influence of global lighting changes and produces more coherent responses to dynamic indirect illumination over time.

\subsection{Ablation Studies}
\label{sec:exp-ablation}

We ablate the injection branches and auxiliary losses on the held-out benchmark. Each variant is trained under the same data, resolution, and inference frame budget as the full model.

\begin{table}[t]
\caption{\textbf{Ablation study of the edit feature injection module.} The three branches are complementary: cross-attention is strongest for local routing, FiLM improves global appearance control, and the additive branch contributes a stable residual edit bias.}
\vspace{3mm}
\label{tab:ablation-branch}
\centering
\small
\begin{tabular}{lccccc}
\toprule
Variant & Trainable params & CLIP-E$\uparrow$ & DINO-sim$\uparrow$ & Warp$\downarrow$ & FVD$\downarrow$ \\
\midrule
Cross-attention only & 312MB & 0.243 & 0.692 & 3.41 & 921.4 \\
FiLM only & 187MB & 0.218 & 0.704 & 3.55 & 938.2 \\
Additive only & 154MB & 0.196 & 0.681 & 3.78 & 967.5 \\
\textbf{All branches (ours)} & 480MB & \textbf{0.245} & \textbf{0.738} & \textbf{2.97} & \textbf{852.1} \\
\bottomrule
\end{tabular}
\end{table}
\begin{table}[t]
\caption{\textbf{Auxiliary-loss ablation.} Removing $\mathcal{L}_{\mathrm{edit}}$ weakens edit transfer; removing $\mathcal{L}_{\mathrm{DINO}}$ increases structural drift and warp error.}
\vspace{3mm}
\label{tab:ablation-loss}
\centering
\small
\begin{tabular}{lccccc}
\toprule
Variant & CLIP-E$\uparrow$ & DINO-sim$\uparrow$ & Temp.$\uparrow$ & Warp$\downarrow$ & FVD$\downarrow$ \\
\midrule
$\mathcal{L}_{\mathrm{FM}}$ only & 0.184 & 0.671 & 0.943 & 3.86 & 967.4 \\
$\mathcal{L}_{\mathrm{FM}}+\mathcal{L}_{\mathrm{edit}}$ & \textbf{0.273} & 0.682 & 0.948 & 3.61 & 902.1 \\
$\mathcal{L}_{\mathrm{FM}}+\mathcal{L}_{\mathrm{DINO}}$ & 0.198 & 0.731 & 0.961 & 3.04 & 884.6 \\
\textbf{Full objective} & 0.245 & \textbf{0.738} & \textbf{0.964} & \textbf{2.97} & \textbf{852.1} \\
\bottomrule
\end{tabular}
\end{table}
\paragraph{Injection branches.}
Tab.~\ref{tab:ablation-branch} shows that each operation in the feature injection module contributes a distinct capability. Cross-attention is the strongest individual branch for edit fidelity, reaching 0.243 CLIP-E, because it can route the edit feature to spatially relevant tokens and therefore suits local edits particularly well. FiLM obtains higher DINO-sim than cross-attention but lower CLIP-E, which matches its role as a global modulator: it preserves layout and overall scene consistency well but has limited spatial selectivity. The additive branch is weakest alone, yet it provides a useful low-rank residual bias when combined with the other branches. The full gated fusion improves CLIP-E by 0.002 over cross-attention alone and lowers FVD by 69.3. These results are consistent with complementary conditioning effects, although this comparison does not isolate the effect of the increased trainable parameter budget.

\vspace{-1.5mm}
\paragraph{Auxiliary losses.}
Tab.~\ref{tab:ablation-loss} also separates two failure modes. Flow matching alone gives a smooth video prior but weak pressure to realize the demonstrated edit, resulting in low CLIP-E. Adding $\mathcal{L}_{\mathrm{edit}}$ raises CLIP-E from 0.184 to 0.273, indicating that the Stage~I descriptor provides a useful target for video-level edit alignment. However, DINO-sim and Warp improve only modestly, so edit alignment alone does not fully prevent structural drift. Conversely, adding only $\mathcal{L}_{\mathrm{DINO}}$ improves preservation and temporal coherence but leaves edit fidelity low. The full objective achieves the best DINO-sim, Temp., Warp, and FVD scores, with a CLIP-E score of 0.245 compared with 0.273 for edit alignment alone. This indicates a trade-off between edit strength and structural preservation. Qualitative examples in Figure~\ref{fig:app-ablation} illustrate the effects of removing branches or alignment losses.

\vspace{-1.5mm}

\paragraph{Summary of experimental evidence.}
The results show a consistent pattern. Strong video priors can be smooth but under-edit; strong image or inversion control can edit individual frames but lose temporal stability; reference-conditioned methods improve appearance transfer but may copy irrelevant details. \method{} improves this trade-off by representing the user intent as a compact transformation descriptor and injecting it directly into the denoising dynamics of a frozen video backbone, enabling controllable edits with temporal coherence and structural consistency.

\section{Conclusion}
\label{sec:conclusion}

In this paper, we propose \method, a two-stage framework for transformation-guided video editing that distills image-level editing description into a compact feature and injects it into a video diffusion model. 
By guiding the denoising process with this feature, together with edit-descriptor alignment and DINO structural consistency losses for temporal consistency, our method enables controllable video editing while preserving the generative prior of the backbone.
Our framework provides a practical paradigm for integrating image-level editing into video generation, with potential applications in content creation, video post-production, and interactive media editing.

\bibliographystyle{plainnat}
\bibliography{nips}

\appendix
\newpage
\section{Limitations and Impact}
\label{sec:impact}

\method{} still has several limitations. Because the descriptor is injected implicitly into the denoising process, highly localized edits may slightly affect nearby regions when attention spreads beyond the intended area. Since the descriptor is extracted from the first frame, performance can degrade under very large motion, severe occlusion, or scene changes that make the first-frame edit poorly representative of later frames. 

This work enables transformation-guided video editing by transferring image-level edits to videos, with applications in creative content production, video post-production, and assistive visual tools. It also supports controllable simulation of environmental changes such as lighting and weather.

At the same time, improved editing capability may increase the risk of deceptive media generation, identity manipulation, and unauthorized content modification. By constraining edit specification through visual exemplars, the method may lower the barrier for targeted misuse.

This work does not involve human-subject experiments. Responsible use should emphasize usage restrictions, provenance tracking, and watermarking or other disclosure mechanisms. In addition, training data and exemplars should respect privacy, copyright, and licensing constraints. We recommend treating portrait editing and safety-sensitive scenarios as restricted deployment settings.

\section{Stage I Architectural Details}
\label{app:stage1-details}

\begin{figure*}[h]
    \centering
    \includegraphics[width=1.0\linewidth]{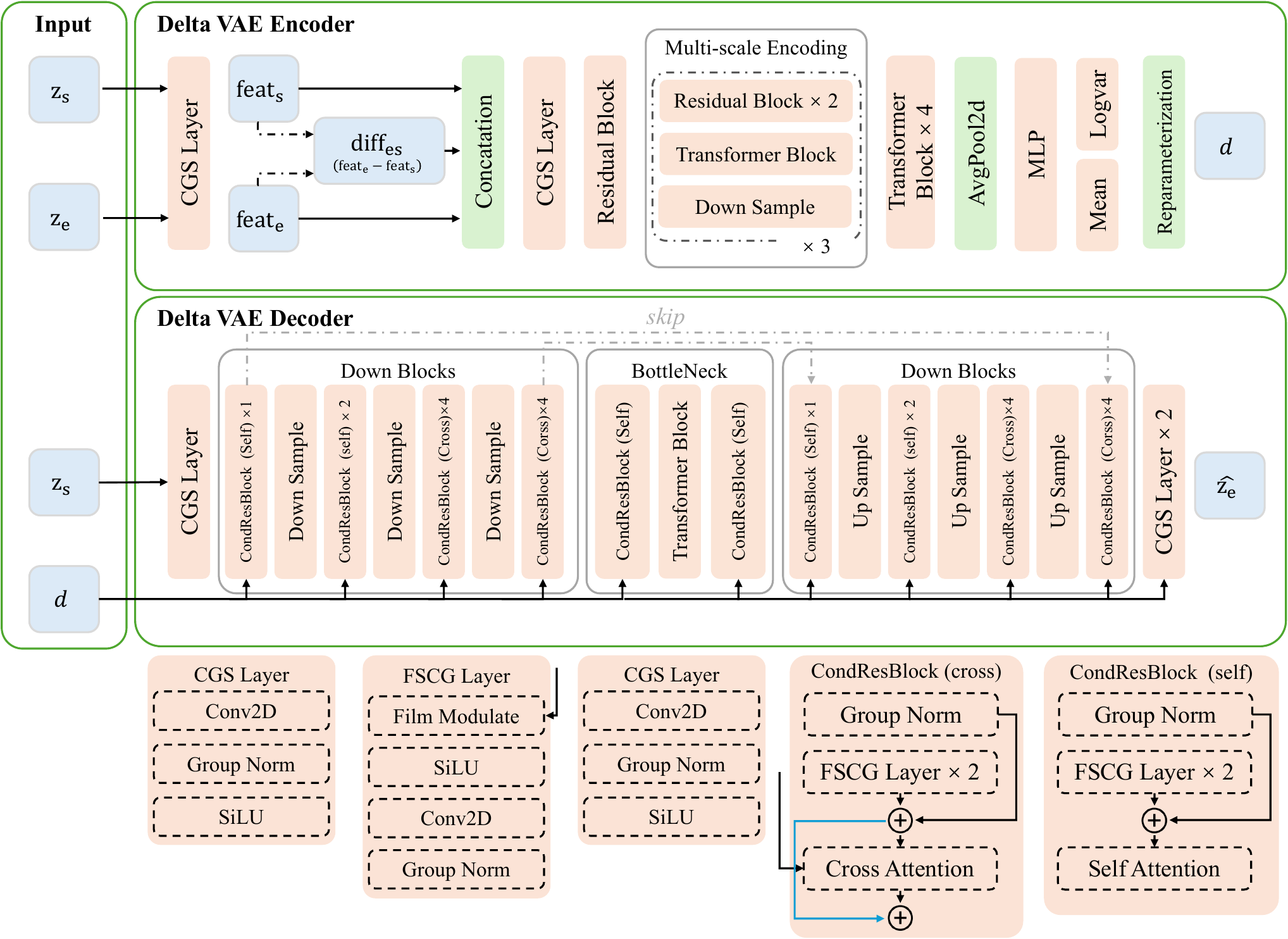}
    \caption{\textbf{Stage~I edit descriptor model.} The Delta VAE encoder extracts a stochastic descriptor from the source and edited VAE latents by combining shared features, their explicit feature difference, and a multi-scale encoder. The Delta VAE decoder receives only the source latent and descriptor, and reconstructs the edited latent through descriptor-conditioned residual blocks with self- and cross-attention.}
    \label{fig:app-delta-vae}
\end{figure*}

Stage I is implemented as the Delta VAE shown in Figure~\ref{fig:app-delta-vae}. It operates entirely in the frozen Wan VAE latent space: the source and edited images are encoded as $z^s,z^e\in\mathbb{R}^{4\times H\times W}$. The encoder observes both latents and compresses the edit into a 1024-dimensional stochastic descriptor $\mathbf{d}$, whereas the decoder observes only $z^s$ and $\mathbf{d}$. This asymmetric design is important: the edited target is available when extracting the descriptor, but it is never passed to the reconstruction decoder through skip connections. As a result, target-specific information must pass through the compact descriptor rather than being copied spatially from $z^e$.

\paragraph{Latent-delta encoder.}
The encoder $E_{\phi}$ extracts the transformation from the pair $(z^s,z^e)$. As illustrated in the upper part of Figure~\ref{fig:app-delta-vae}, the two latents first pass through a shared CGS stem, where each CGS layer consists of convolution, group normalization, and SiLU activation:
\begin{equation}
    f^s=S(z^s), \qquad f^e=S(z^e),
\end{equation}
so the source and edited features are represented in a comparable basis. We then form an explicit feature-level edit residual
\begin{equation}
    \Delta f=f^e-f^s,
\end{equation}
and concatenate $[f^s,f^e,\Delta f]$ before another CGS projection and a residual block. This explicit delta path makes the encoder focus on what changed, while still retaining the source and target context needed to disambiguate edits such as object insertion, material replacement, relighting, and stylization.

The projected feature is processed by a multi-scale hierarchy. Each scale contains two residual blocks, a transformer block, and a downsampling operation; this sequence is repeated three times. Local residual blocks preserve fine edit evidence, the downsampling path enlarges the receptive field, and transformer blocks allow non-local interactions between changed and unchanged regions. A final stack of four transformer blocks aggregates the coarsest representation before adaptive average pooling removes spatial coordinates. An MLP then predicts the mean and log-variance of a diagonal Gaussian descriptor distribution:
\begin{equation}
    q_{\phi}(\mathbf{d}\mid z^s,z^e)=\mathcal{N}\left(\mu_{\phi}(z^s,z^e),\operatorname{diag}(\sigma_{\phi}^{2}(z^s,z^e))\right),
    \qquad \mathbf{d}\in\mathbb{R}^{1024}.
\end{equation}
During training, $\mathbf{d}$ is sampled using the reparameterization trick. At inference time, we use the posterior mean as a deterministic descriptor.

\paragraph{FiLM-conditioned residual decoder.}
The decoder $D_{\psi}$ reconstructs the edited latent from $(z^s,\mathbf{d})$ only:
\begin{equation}
    \hat{z}^e=z^s+D_{\psi}(z^s,\mathbf{d}).
\end{equation}
The source latent first passes through a CGS layer and then a U-Net-like residual decoder with down blocks, a bottleneck, up blocks, and skip connections from matching resolutions. The skip connections carry source layout and identity, but only from $z^s$; no feature from $z^e$ enters the decoder.

The descriptor conditions every decoder stage through FSCG layers and conditional residual blocks. An FSCG layer applies descriptor-derived FiLM modulation, SiLU activation, convolution, and group normalization. In the down and up paths, conditional residual blocks use self-attention to preserve source structure and cross-attention to route descriptor information to spatial positions that require the edit. The bottleneck combines descriptor-conditioned residual processing with a transformer block, giving the decoder a global view before upsampling. The final CGS head predicts a four-channel residual in VAE latent space, which is added to $z^s$ to obtain $\hat{z}^e$. This residual formulation lets the decoder leave unchanged regions close to the source while concentrating descriptor capacity on the edit itself.

\paragraph{Bottleneck capacity.}
We use a 1024-dimensional descriptor by default. Smaller descriptors lose fine edit attributes such as local insertions, material changes, and subtle lighting transformations. Larger descriptors improve reconstruction only marginally but are more prone to source-specific leakage. The chosen capacity balances reconstruction fidelity and transferability.

\paragraph{Stage~II implication.}
Because Stage~I forces the edit through a global descriptor and applies it through conditional modulation, $\mathbf{d}$ behaves as an edit operator rather than a spatial template. Stage~II can therefore inject the same descriptor into video transformer blocks and let the video features determine where the edit should appear in each frame. This is why the same image-level descriptor can re-localize to moving regions over time, as shown in Figure~\ref{fig:app-attention-video}.

\section{Full Flow-Matching Derivation for Stage~II}
\label{app:fm-derivation}

\subsection{Forward Interpolation Path}

The forward process is defined as a straight-line interpolation between a data
sample $x_0$ and a noise vector $\varepsilon$:

\begin{equation}
    x_t = (1 - t)\,x_0 + t\,\varepsilon,
    \qquad \varepsilon \sim \mathcal{N}(\mathbf{0},\mathbf{I}),
    \quad t \in [0,1].
    \label{eq:forward_interp}
\end{equation}

Here $t=0$ corresponds to clean data and $t=1$ to pure noise.

\subsection{Training Objective}

The model is trained to regress the \emph{true velocity} along the interpolation path:

\begin{equation}
    \Lfm = \mathbb{E}_{t,\,x_0,\,\varepsilon}\!\left[
        \bigl\| \vfield(x_t,\,t,\,c) - (\varepsilon - x_0) \bigr\|^2
    \right],
    \label{eq:fm_loss}
\end{equation}

where $c$ is the conditioning signal.
\subsection{Reverse ODE for Inference}

Given a trained vector field, new samples are generated by solving the reverse ODE:

\begin{equation}
    \frac{dx_t}{dt} = \vfield(x_t,\,t,\,c)= \varepsilon - x_0,
    \qquad x_1 \sim \mathcal{N}(\mathbf{0},\mathbf{I}) ,
    \label{eq:reverse_ode}
\end{equation}

integrating from $t=1$ to $t=0$.
In practice, this is discretized with an Euler solver or DPM-Solver.

\subsection{Notation}

\begin{center}
\renewcommand{\arraystretch}{1.4}
\begin{tabular}{cl}
    \toprule
    Symbol & Meaning \\
    \midrule
    $I^s$ & Source reference image (e.g.\ a tiger without glasses) \\
    $I^e$ & Edited reference image (e.g.\ the same tiger wearing glasses) \\
    $V^s = \{v_n^s\}_{n=1}^{N}$ & Source video of $N$ frames \\
    $V^e = \{v_n^e\}_{n=1}^{N}$ & Target: edited video (tiger wears glasses in every frame) \\
    $\xV \in \R^{N \times H \times W \times C}$ & Noisy video latent at time $t$ \\
    $E_{\phi}$ & Delta VAE Encoder \\
    $\mathrm{DINO}(\cdot)$ & Self-supervised DINO feature extractor \\
    $c$ & Conditional of base video model \\
    \bottomrule
\end{tabular}
\end{center}

\subsection{Edit Consistency Assumption}

We first define the image-level \textbf{edit descriptor}:

\begin{equation}
    \DeltaF \;\triangleq\; E_{\phi}(I^s, I^e)
    \;\in\; \R^{B'\!\times d}.
    \label{eq:edit_delta}
\end{equation}

\begin{definition}[Edit Consistency]
    The editing transformation encoded in $\DeltaF$ should be \emph{consistently}
    applied to every frame of the video. Formally, for each frame $n$:
    \begin{equation}
        E_{\phi}(\hat{v}_n^s, \hat{v}_n^e)
        \;\approx\; \DeltaF,
        \label{eq:consistency_assumption}
    \end{equation}
    where $\hat{v}_n^s$ and $\hat{v}_n^e$ denote the $n$-th frame of the original
    and edited video predictions, respectively.
\end{definition}

\subsection{Additive Decomposition of the Vector Field}

We decompose the editing vector field into an \emph{unconstrained base prediction}
and an \emph{edit guidance correction}:

\begin{equation}
    \vfedit(\xV,\,t)
    = \underbrace{\vfield(\xV,\,t,\,c)}_{\text{base prediction}}
    \;+\; \underbrace{\Delta v_t^{\mathrm{guidance}}}_{\text{guidance correction}}.
    \label{eq:vfield_decomp}
\end{equation}

\subsection{Score Function and Vector Field Correspondence}
\begin{proposition}[Score--Velocity Duality]
    Under the linear interpolation path~\eqref{eq:forward_interp}, the score
    function of the marginal $p_t$ satisfies:
    \begin{equation}
        \nabla_{x_t} \log p_t(x_t)
        = -\frac{\vfield(x_t,\,t)}{t}.
        \label{eq:score_velocity}
    \end{equation}
\end{proposition}

\begin{proof}
From~\eqref{eq:forward_interp}, the conditional distribution is
$p_t(x_t\mid x_0) = \mathcal{N}\!\left(x_t;\,(1-t)x_0,\,t^2\mathbf{I}\right)$,
so
\[
    \nabla_{x_t}\log p_t(x_t\mid x_0)
    = -\frac{x_t - (1-t)x_0}{t^2}
    = -\frac{\varepsilon - x_0}{t}.
\]
Taking the conditional expectation over $x_0$ given $x_t$ yields Eq.~\ref{eq:score_velocity}.
\end{proof}

\subsection{Bayesian Guidance Framework}

Conditioning on additional edit descriptor features $F^e_t$, the joint score function
decomposes as:

\begin{equation}
    \nabla_{\xV} \log p_t\!\left(\xV,\,F^e_t\right)
    = \nabla_{\xV}\log p_t(\xV)
    + \nabla_{\xV}\log p_t\!\left(F^e_t \mid \xV\right).
    \label{eq:score_decomp}
\end{equation}

Substituting~\eqref{eq:score_velocity} into~\eqref{eq:score_decomp} and converting
back to vector field language yields the \textbf{guided vector field}:

\begin{equation}
    \boxed{
        \vfedit(\xV,\,t)
        = \vfield(\xV,\,t,\,c)
        \;-\; t\cdot\nabla_{\xV}\,\Lguide(\xV,\,t)
    }
    \label{eq:vfield_guided}
\end{equation}

\subsection{Edit Feature Alignment Loss \texorpdfstring{$\Ledit$}{L\_edit}}

\textbf{Step 1.} Compute the image edit residual (computed \emph{once}, frozen):

\begin{equation}
    \DeltaF \;\triangleq\; E_{\phi}(I^s, I^e)
    \;\in\; \R^{B'\!\times d}.
    \label{eq:delta_F}
\end{equation}

\textbf{Step 2.} Enforce the edit consistency assumption~\eqref{eq:consistency_assumption}
across all $N$ frames:

\begin{equation}
    \Ledit
    = \sum_{n=1}^{N}
        \left\|\,
            E_{\phi}(\hat{v}_n^s, \hat{v}_n^e)
            - \DeltaF
        \,\right\|_2^2,
    \label{eq:L_edit}
\end{equation}

where $\hat{v}_n^s$ and $\hat{v}_n^e$ denote the $n$-th frame
of $\Vs$ and $\Ve$, respectively.

\begin{remark}
    Equation~\eqref{eq:L_edit} enforces that the per-frame feature shift induced
    by the editing process equals the image-level editing direction $\DeltaF$.
    In other words, the edit semantics (``wearing glasses'') are propagated
    consistently to every frame without requiring any frame-specific supervision.
\end{remark}

\subsection{Temporal DINO Consistency Loss \texorpdfstring{$\Ldino$}{L\_DINO}}

To preserve the unedited structure of the video (body motion, background, etc.),
we align the DINO features of the predicted frames to those of the original video:

\begin{equation}
    \Ldino
    = \sum_{n=1}^{N}
        \left\|\,
            \mathrm{DINO}\!\bigl(\hat{v}_n^e)
            - \mathrm{DINO}(v_n^s)
        \,\right\|_2^2,
    \label{eq:L_dino}
\end{equation}

where $\mathrm{DINO}(v_n^s)$ is the DINO feature of the $n$-th original frame,
computed \textbf{once} and kept frozen during inference.

\begin{remark}
    DINO features are sensitive to \emph{semantic structure} (pose, layout, shape)
    but relatively insensitive to local appearance changes (e.g.\ color overlays
    or accessory additions). This makes them an ideal regularizer: they penalize
    structural drift without suppressing the desired edit.
\end{remark}

\subsection{Total Guidance Loss}

Combining both terms:
\begin{equation}
    \boxed{
        \Lguide
        = \alpha_{\mathrm{edit}}\cdot\Ledit
        + \alpha_{\mathrm{DINO}}\cdot\Ldino
    }
    \label{eq:L_total}
\end{equation}
where $\alpha_{\mathrm{edit}},\,\alpha_{\mathrm{DINO}}>0$ are balancing
hyperparameters.

\subsection{Complete Guided Vector Field}

Substituting~\eqref{eq:L_total} into~\eqref{eq:vfield_guided}, the full
\textbf{image-guided video editing vector field} is:

\begin{equation}
    \boxed{
        \vfedit(\xV,\,t)
        = \underbrace{\vfield(\xV,\,t,\,c^A)}_{\text{base denoising direction}}
        \;-\; t\!\left(
            \underbrace{
                \alpha_{\mathrm{edit}}\,\nabla_{\xV}\Ledit
            }_{\text{edit feature alignment}}
            +
            \underbrace{
                \alpha_{\mathrm{DINO}}\,\nabla_{\xV}\Ldino
            }_{\text{temporal consistency}}
        \right)
    }
    \label{eq:app-guided-vfield}
\end{equation}

\begin{remark}
    The two guidance terms play complementary roles:
    \begin{itemize}[itemsep=2pt]
        \item \textbf{Edit alignment} steers the denoising trajectory toward
              the edited state, injecting the image-level edit direction
              $I^s\!\to\!I^e$ into every video frame.
        \item \textbf{DINO consistency} acts as a regularizer, preventing
              structural drift in parts of the video that should remain
              unchanged (e.g.\ tiger's body, motion, background).
    \end{itemize}
\end{remark}

\section{Inference-Time Guidance Algorithm}
\label{app:guidance-algorithm}

Algorithm~\ref{alg:guidance} describes optional inference-time guidance. The trained transformation-conditioned field $v_{\Theta}$ serves as the base velocity. At each step, we compute a one-step clean prediction,
\begin{equation}
    \hat{h}_0=h_t-t\,v_{\Theta}(h_t,t,c,\mathbf{d}),
    \label{eq:app-onestep}
\end{equation}
measure edit and DINO losses on that prediction, and correct the velocity by the gradient in Equation~\ref{eq:app-guided-vfield}. Gradient clipping prevents rare feature-space outliers from destabilizing the trajectory.

\begin{algorithm}[H]
\caption{Inference-time guidance for \method.}
\label{alg:guidance}
\begin{algorithmic}[1]
\Require Source video $V^s=\{v_n^s\}_{n=1}^N$, image pair $(I^s,I^e)$, edit descriptor feature $\mathbf{d}$, Wan conditions $c=(c_{\mathrm{text}},c_{\mathrm{img}})$, Euler steps $T$, weights $\alpha_{\mathrm{edit}},\alpha_{\mathrm{DINO}}$, gradient clip $\lambda_{\max}$.
\Ensure Edited video $V^e$.
\Initialize
\State $h_1\sim\mathcal{N}(0,I)$
\State Precompute $\{\Psi(v_n^s)\}_{n=1}^{N}$ with frozen DINO
\For{$k=T,T-1,\ldots,1$}
    \State $t\leftarrow k/T$, $\Delta t\leftarrow 1/T$
    \State $v_{\mathrm{base}}\leftarrow v_{\Theta}(h_t,t,c,\mathbf{d})$
    \State $\hat{h}_0\leftarrow h_t-t\,v_{\mathrm{base}}$
    \State Construct source reference $\hat{h}_0^s$ by noising $V^s$ to time $t$ and denoising without $\mathbf{d}$
    \State $\mathcal{L}_{\mathrm{edit}}\leftarrow \frac{1}{N}\sum_n\|E_{\phi}(\hat{h}_{0,n}^s,\hat{h}_{0,n})-\mathbf{d}\|_2^2$
    \State $\mathcal{L}_{\mathrm{DINO}}\leftarrow \frac{1}{N}\sum_n\|\Psi(\hat{h}_{0,n})-\Psi(v_n^s)\|_2^2$
    \State $g\leftarrow\nabla_{h_t}(\alpha_{\mathrm{edit}}\mathcal{L}_{\mathrm{edit}}+\alpha_{\mathrm{DINO}}\mathcal{L}_{\mathrm{DINO}})$
    \State $g\leftarrow g\cdot\min(1,\lambda_{\max}/\|g\|_2)$
    \State $v_{\mathrm{guided}}\leftarrow v_{\mathrm{base}}-t\,g$
    \State $h_{t-\Delta t}\leftarrow h_t-\Delta t\,v_{\mathrm{guided}}$
\EndFor
\State $V^e\leftarrow\mathcal{D}_{\mathrm{vae}}(h_0)$
\State \Return $V^e$
\end{algorithmic}
\end{algorithm}

\section{Edit Descriptor Feature Injection Modules: Per-Branch Details}
\label{app:injector-details}

\begin{figure*}[t]
    \centering
    \centerline{\includegraphics[width=1.0\linewidth]{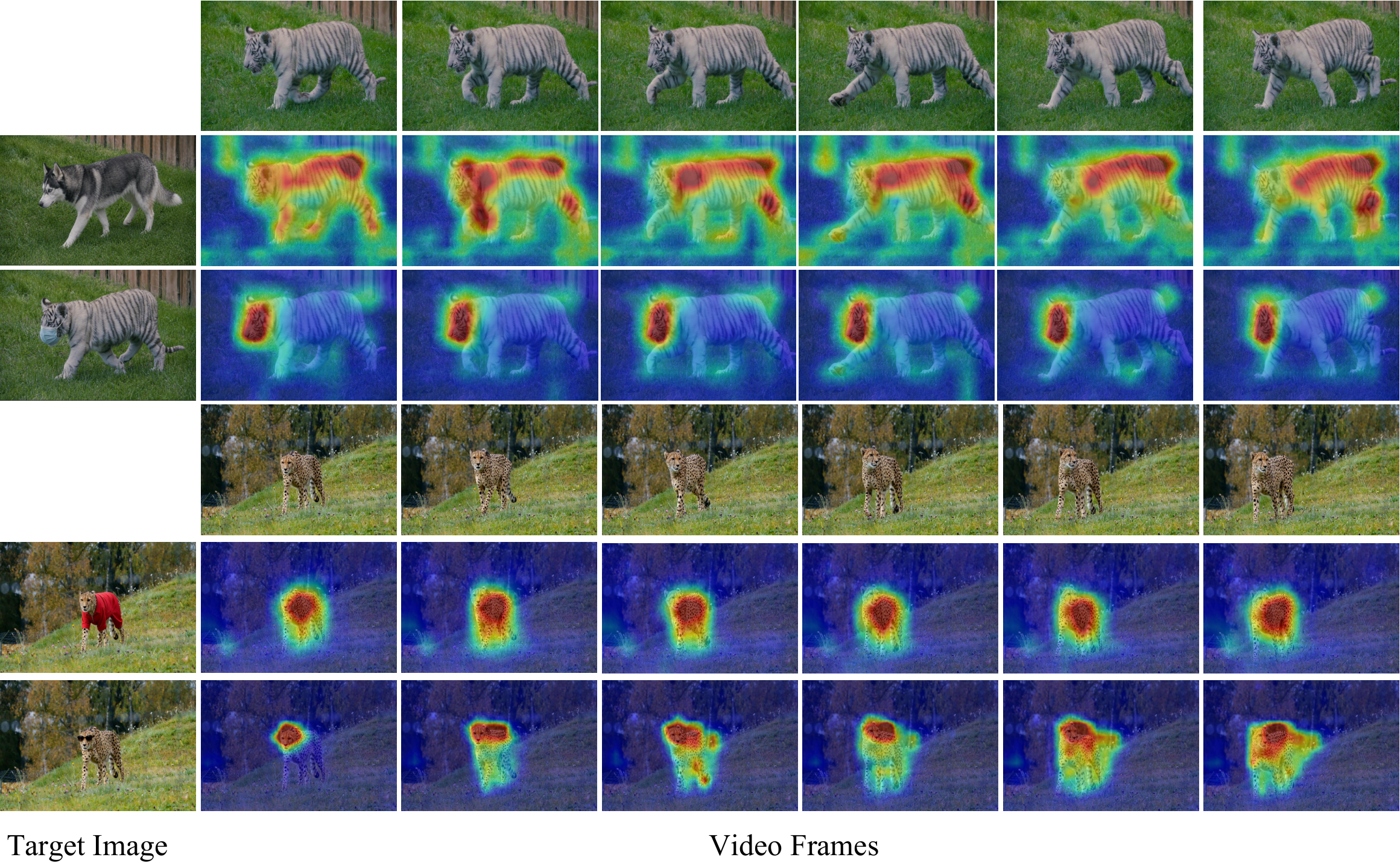}}
    \caption{\textbf{Frame-wise descriptor attention.} The same edit descriptor re-localizes to the relevant moving regions across different video frames.}
    \vspace{-1.3em}
    \label{fig:app-attention-video}
\end{figure*}

\begin{figure*}[t]
    \centering
    \centerline{\includegraphics[width=1.0\linewidth]{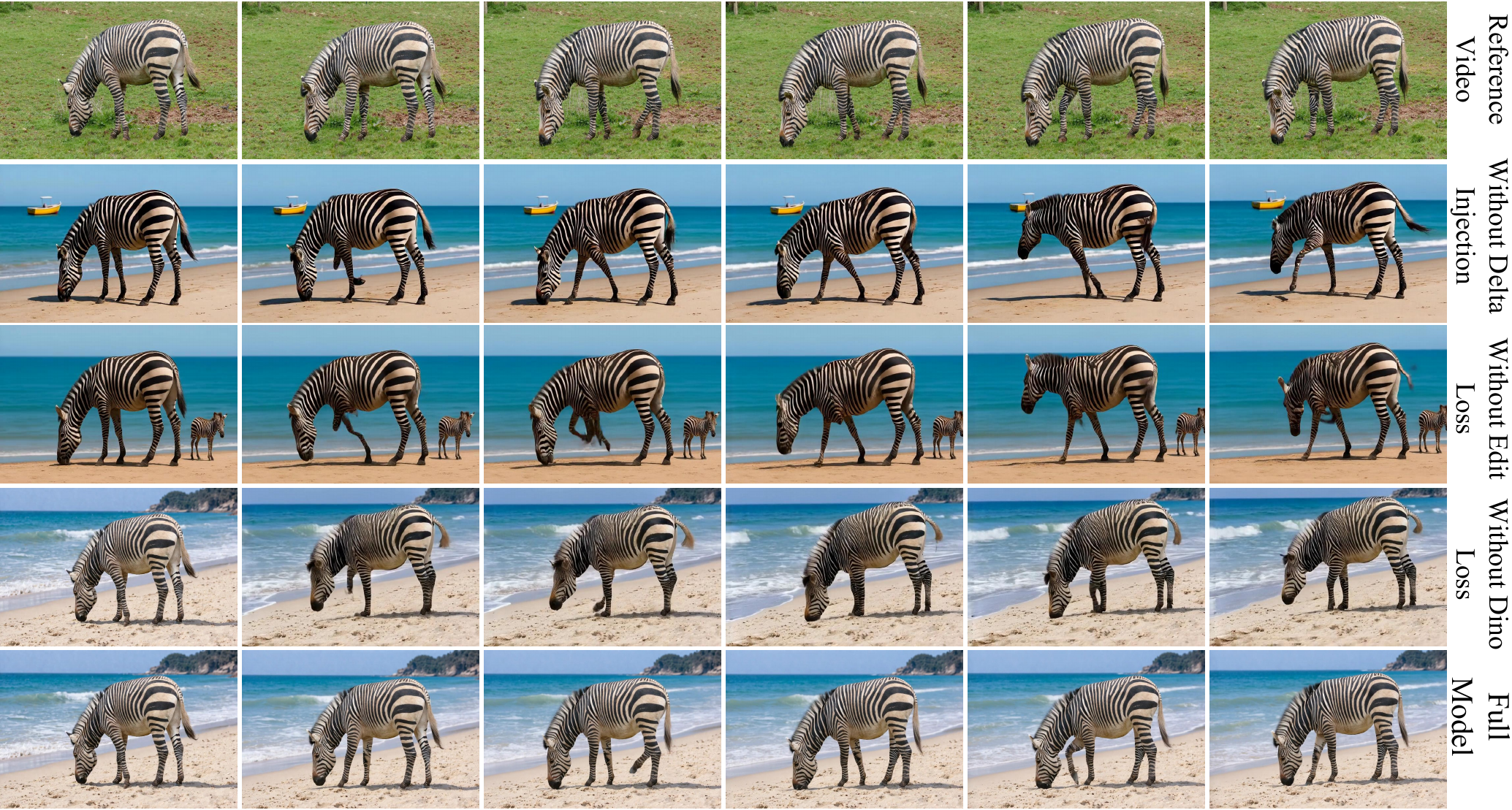}}
    \caption{\textbf{Qualitative ablation.} Removing individual branches or alignment losses weakens either edit fidelity or temporal preservation, matching the quantitative trends in Section~\ref{sec:exp-ablation}.}
    \vspace{-1.3em}
    \label{fig:app-ablation}
\end{figure*}

This appendix expands the compressed description in Section~\ref{sec:stage2} (``Delta injection modules'') into the full per-branch design, with the rationale for each branch and the layer-selection policies that we sweep in the ablations.

\paragraph{Shared MLP stem.}
Before any injection, every injector first normalizes and adapts the latent code through a shared MLP stem,
\begin{equation}
\bar{d} \;=\; \mathrm{MLP}\bigl(\mathrm{LN}(d)\bigr) + d.
\label{eq:app-mlp-stem}
\end{equation}
The shared stem isolates the basis change between the Stage~I latent space (1024-dimensional, optimized for image-pair reconstruction) and the per-block residual stream of the video DiT (4096--6144-dimensional depending on the block, optimized for spatiotemporal denoising). Sharing this stem across branches keeps the trainable parameter count small and acts as a soft regularizer that prevents each branch from learning a private, conflicting interpretation of $d$.

\paragraph{Branch 1: cross-attention.}
The cross-attention branch lets the hidden sequence $h \in \mathbb{R}^{T \times D}$ attend to the adapted code $\bar{d}$ as a global edit token,
\begin{equation}
h_{\mathrm{cross}} \;=\; h \;+\; \mathrm{CrossAttn}\!\bigl(Q(h),\, K(\bar{d}),\, V(\bar{d})\bigr).
\label{eq:app-cross}
\end{equation}
Because the spatial query side decides where to attend on a per-token basis, this branch is the natural fit for \emph{spatially selective} edits (object-local color changes, identity swaps in a face region, accessory addition). It also inherits the spatial re-localization property of the Stage~I attention: the same $\bar{d}$ produces a different attention mask on each frame, so the edit follows the moving content.

\paragraph{Branch 2: FiLM.}
The FiLM branch \citep{perez2018film} uses $\bar{d}$ to predict per-channel scale and shift parameters that modulate $h$ globally,
\begin{equation}
h_{\mathrm{film}} \;=\; \gamma(\bar{d}) \odot h \;+\; \beta(\bar{d}), \qquad \gamma, \beta : \mathbb{R}^{|d|} \to \mathbb{R}^{D}.
\label{eq:app-film}
\end{equation}
FiLM has no spatial selectivity (the same $\gamma, \beta$ are applied at every token), which is precisely what we want for \emph{global appearance transforms} such as global color grade, time-of-day shift, or stylization. We initialize $\gamma$ to one and $\beta$ to zero, so the FiLM branch is identity at the start of training.

\paragraph{Branch 3: additive.}
The additive branch projects $\bar{d}$ into a small set of $L_a$ learned token slots and fuses them into the residual stream through a low-rank residual path,
\begin{equation}
h_{\mathrm{add}} \;=\; h \;+\; W_{\mathrm{up}}\,\sigma\!\bigl(W_{\mathrm{down}}[\,h \,;\, \bar{d}\,]\bigr),
\label{eq:app-additive-injector}
\end{equation}
with $W_{\mathrm{down}} \in \mathbb{R}^{r \times (D + |d|)}$, $W_{\mathrm{up}} \in \mathbb{R}^{D \times r}$, $r \ll D$, and $\sigma$ a SiLU activation. This is the most expressive branch but also the most prone to overfitting; it acts as a learned low-rank edit bias and is the most useful one when the demonstrated edit cannot be cleanly described as either a spatial overlay (cross-attention) or a global recolor (FiLM).

\paragraph{Gated fusion.}
Branch outputs are combined by a learned, scalar-gated projection,
\begin{equation}
h_{\mathrm{out}} \;=\; h \;+\; \mathrm{tanh}(\xi)\cdot \bigl[\, h_{\mathrm{cross}} - h \;+\; h_{\mathrm{film}} - h \;+\; h_{\mathrm{add}} - h \,\bigr],
\label{eq:app-fusion}
\end{equation}
with $\xi$ a learnable scalar initialized to zero. At initialization $\mathrm{tanh}(\xi) = 0$, so $h_{\mathrm{out}} = h$ and the wrapped backbone is bit-identical to the pre-trained Wan2.1-I2V-14B; the gate then opens monotonically during training. This is the standard zero-init trick \citep{zhang2023controlnet} that makes fine-tuning of large pre-trained generators stable.

\begin{figure*}[t]
    \centering
    \centerline{\includegraphics[width=1.0\linewidth]{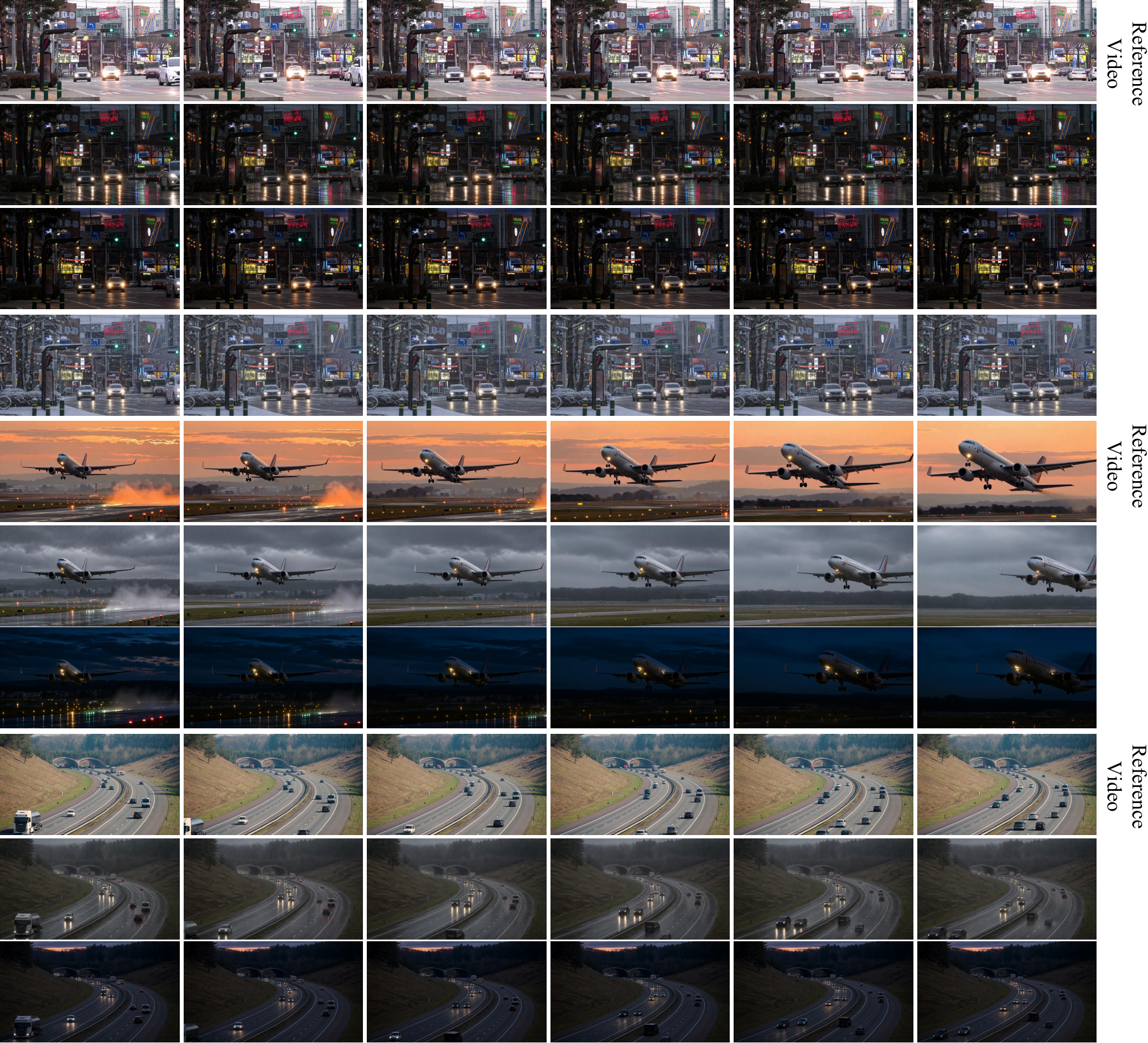}}
    \caption{\textbf{More results in traffic environments.}}
    \vspace{-1.3em}
    \label{fig:app-traffic-results}
\end{figure*}

\begin{figure*}[t]
    \centering
    \centerline{\includegraphics[width=1.0\linewidth]{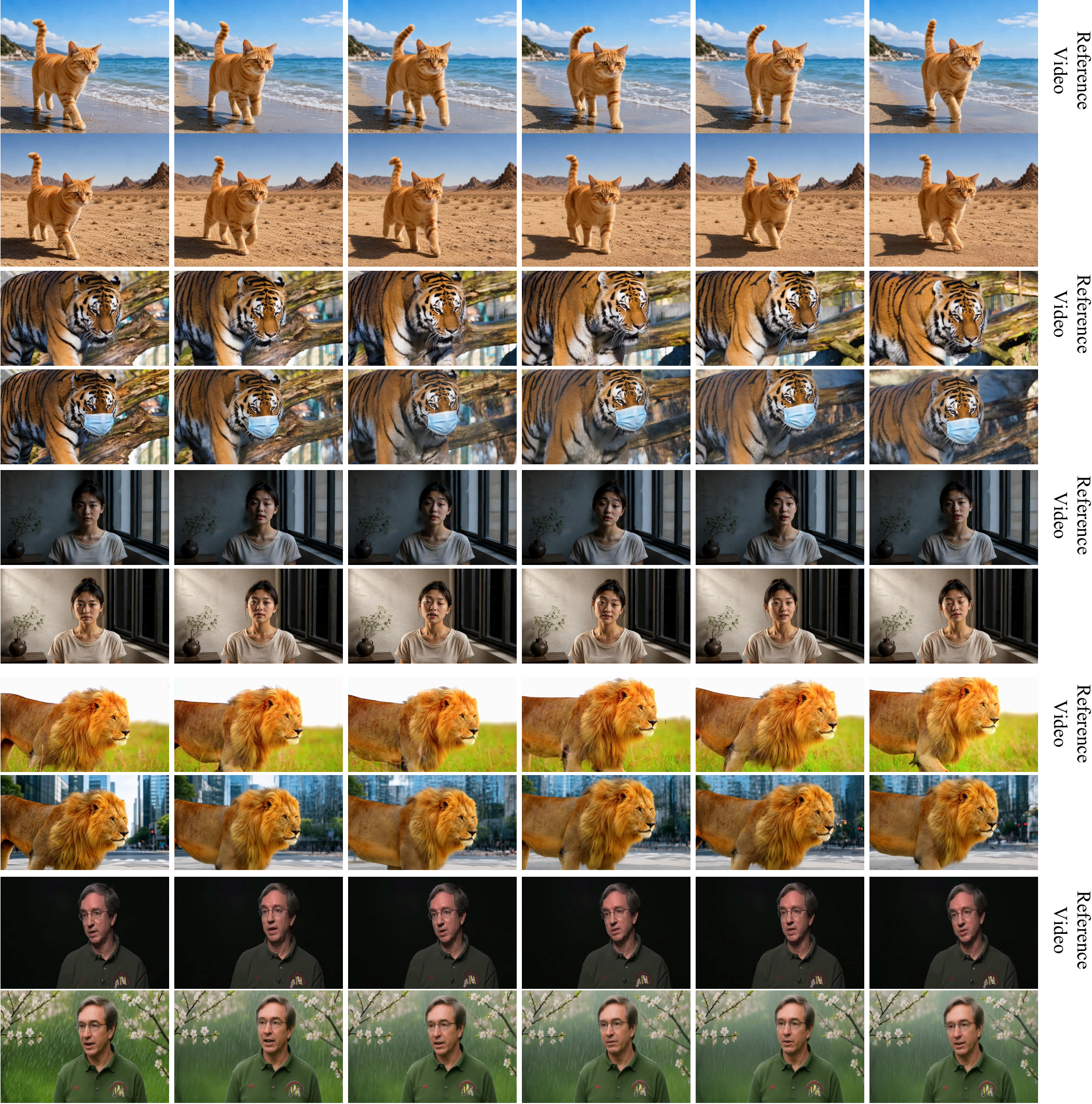}}
    \caption{\textbf{Additional video editing results.}}
    \vspace{-1.3em}
    \label{fig:app-more-results}
\end{figure*}

\end{document}